\pdfoutput=1

\documentclass[letterpaper, 10pt, conference, dvipsnames]{ieeeconf}  %
\IEEEoverridecommandlockouts                              %

\usepackage[basicpaper,lengthhacks]{tinyiu}
\usepackage{capt-of}%
\usepackage[boxed,ruled,vlined,linesnumbered]{algorithm2e}
\DontPrintSemicolon
\SetKwProg{Fn}{function}{}{}
\SetKwFunction{FnGetAction}{GetAction}
\SetKwFunction{FnDelegatePolicy}{DelegatePolicy}
\SetKwFunction{FnRestartArm}{RestartArm}
\SetKwFunction{FnStep}{Step}
\SetKwComment{Comment}{$\triangleright$\ }{}
\SetKwInput{KwInit}{Initialise}

\SetKwFunction{FnPlanTrajectory}{PlanTrajectory}

\usepackage{etoolbox}
\makeatletter
\patchcmd{\@algocf@start}%
  {-1.5em}%
  {0pt}%
  {}{}%
\makeatother

\usepackage{wrapfig}
\usepackage{minibox}

\usepackage[inline]{enumitem}

\usepackage{subcaption}

\usepackage{booktabs}
\usepackage{multirow}

\usepackage{amssymb}                                        %
\usepackage{mathtools}
\usepackage{braket}                                         %

\usepackage{graphicx}                                       %
\usepackage{tikz}

\usepackage{bm}
\usepackage{dsfont}
\usepackage{xifthen}
\usepackage{stackengine} %
\usepackage{url}

\usepackage{cleveref}                                        %
\crefname{assumption}{assumption}{assumptions}
\crefname{problem}{problem}{problems}
\crefname{algorithm}{Alg.}{Algs.}
\Crefname{algorithm}{Algorithm}{Algorithms}
\crefname{figure}{Figure}{Figs.} %
\crefformat{equation}{(#2#1#3)}
\crefrangeformat{equation}{(#3#1#4) to~(#5#2#6)}
\crefmultiformat{equation}{(#2#1#3)}%
{ and~(#2#1#3)}{, (#2#1#3)}{ and~(#2#1#3)}

\crefformat{footnote}{#2\footnotemark[#1]#3}

\usepackage[boxed,ruled,vlined,linesnumbered]{algorithm2e}
\DontPrintSemicolon
\SetKwProg{Fn}{function}{}{}
\SetKwFunction{FnSampleFree}{SampleFree}
\SetKwFunction{FnRestartArm}{RestartArm}
\SetKwFunction{FnPickArm}{PickArm}
\SetKwFunction{FnRewire}{Rewire}
\SetKwFunction{FnNearest}{Nearest}
\SetKwFunction{FnRestartArm}{RestartArm}
\SetKwComment{Comment}{$\triangleright$\ }{}
\SetKwInput{KwInit}{Initialise}

\SetCommentSty{commentfont}

\definecolor{amethyst}{rgb}{0.6, 0.4, 0.8}
\definecolor{alizarin}{rgb}{0.82, 0.1, 0.26}
\definecolor{ashgrey}{rgb}{0.43, 0.5, 0.5}
\definecolor{yellow}{rgb}{1.0, 0.75, 0.0} %

\newcommand*{\norm}[1]{\lVert#1\rVert}

\makeatletter
\newcommand{\overrightsmallarrow}{\mathpalette{\overarrowsmall@\rightarrowfill@}}
\newcommand{\overarrowsmall@}[3]{%
  \vbox{%
    \ialign{%
      ##\crcr
      #1{\smaller@style{#2}}\crcr
      \noalign{\nointerlineskip}%
      $\m@th\hfil#2#3\hfil$\crcr
    }%
  }%
}
\def\smaller@style#1{%
  \ifx#1\displaystyle\scriptstyle\else
    \ifx#1\textstyle\scriptstyle\else
      \scriptscriptstyle
    \fi
  \fi
}
\makeatother

\mathchardef\ordinarycolon\mathcode`\:
\mathcode`\:=\string"8000
\begingroup \catcode`\:=\active
  \gdef:{\mathrel{\mathop\ordinarycolon}}
\endgroup

\usepackage{mathtools}

\makeatletter
\newcommand{\@givenstar}{\;\middle|\;}
\newcommand{\@givennostar}[1][]{\;#1|\;}
\newcommand{\given}{\@ifstar\@givenstar\@givennostar}
\makeatother

\newcommand{\dist}[1]{{\parallel #1 \parallel}}

\newcommand*\volof[1]{\mu^{#1}}

\DeclareFontFamily{U} {MnSymbolA}{}
\DeclareFontShape{U}{MnSymbolA}{m}{n}{
  <-6> MnSymbolA5
  <6-7> MnSymbolA6
  <7-8> MnSymbolA7
  <8-9> MnSymbolA8
  <9-10> MnSymbolA9
  <10-12> MnSymbolA10
  <12-> MnSymbolA12}{}
\DeclareFontShape{U}{MnSymbolA}{b}{n}{
  <-6> MnSymbolA-Bold5
  <6-7> MnSymbolA-Bold6
  <7-8> MnSymbolA-Bold7
  <8-9> MnSymbolA-Bold8
  <9-10> MnSymbolA-Bold9
  <10-12> MnSymbolA-Bold10
  <12-> MnSymbolA-Bold12}{}
\DeclareSymbolFont{MnSyA} {U} {MnSymbolA}{m}{n}
\DeclareMathSymbol{\rhookrightarrow}{\mathrel}{MnSyA}{48}

\newcommand{\LTR}{\textsc{ltr}{*}\xspace}
\newcommand{\RRT}{\textsc{rrt}{*}\xspace}
\newcommand{\RRTC}{\textsc{rrt}-{\small Connect}{*}\xspace}
\newcommand{\LazyPRM}{{\small Lazy}-\textsc{prm}{*}\xspace}
\newcommand{\Cspace}{\emph{C-space}\xspace}

\usepackage[
    backend=biber
  ,giveninits=true                  %
  ,url=false, isbn=false, doi=false %
  ,maxnames=99                       %
  ,minnames=3                       %
  ,sorting=none             %
  ,date=year                %
  ,style=ieee
  ]{biblatex}
\AtEveryBibitem{                      %
  \iffieldequalstr{eprinttype}{jstor}
  {\clearfield{eprint}}
  {}
  \clearfield{urlyear}
  \clearfield{urlmonth}
  \clearfield{url}
}

\renewrobustcmd*{\bibinitdelim}{\,} %
\title{Rapid Replanning in Consecutive Pick-and-Place Tasks with Lazy Experience Graph}

\author{Tin Lai$^{\dagger,*}$%
 and Fabio Ramos$^{\dagger,\mathsection}$%
\thanks{
$^{*}$Correspondence to {\tt\small tin.lai@sydney.edu.au}\newline
$^{\dagger}$School of Computer Science, The University of Sydney, Australia.
$^{\mathsection}$NVIDIA, USA.
}%
}

\begin{document}

\maketitle
\thispagestyle{empty}
\pagestyle{empty}

\begin{abstract}
  In an environment where a manipulator needs to execute multiple consecutive tasks%
    , the act of object manoeuvre will change the underlying configuration space, affecting all subsequent tasks.
    Previously free configurations might now be occupied by the manoeuvred objects, and previously occupied space might now open up new paths.
    We propose Lazy Tree-based Replanner (\LTR)---a novel hybrid planner that inherits the rapid planning nature of existing anytime incremental sampling-based planners. At the same time, it allows subsequent tasks to leverage prior experience via a lazy experience graph.
    Previous experience is summarised in a lazy graph structure, and \LTR is formulated to be robust and beneficial regardless of the extent of changes in the workspace.
    Our hybrid approach attains a faster speed in obtaining an initial solution than existing roadmap-based planners and often with a lower cost in trajectory length.
    Subsequent tasks can utilise the lazy experience graph to speed up finding a solution and take advantage of the optimised graph to minimise the cost objective.
    We provide proofs of probabilistic completeness and almost-surely asymptotic optimal guarantees.
    Experimentally, we show that in repeated pick-and-place tasks, \LTR attains a high gain in performance when planning for subsequent tasks.
\end{abstract}

\section{Introduction}
Common robotics applications involve the use of manipulator to manoeuvre objects from some initial locations to target locations~\autocite{lozano-perez1989_TaskPlan}.
Typically industrial robots in automated manufacturing systems~\autocite{perumaal2013_AutoTraj} have tasks composed of reaching, grasping and placing motions, coupled with changing collision geometry during execution due to the attached and moved object~\autocite{yang2018_PlanTime}.
Executing a task will likely
modifies the underlying Configuration Space (\Cspace), which refers to the set of all possible robot configurations~\autocite{elbanhawi2014_SampRobo}.
The modified \Cspace renders its previous motion plans invalid as the previously collision-free space might now be occupied by the moved objects.
Therefore, a motion planner cannot assumes configurations remain valid between each consecutive task.

Sampling-based motion planners (SBPs) are a class of robust methods for motion planning~\autocite{kavraki1996_AnalProb}. In contrast, roadmap-based approaches are capable of querying motion plans for a different set of initial and target configurations~\autocite{karaman2011_SampAlgo} by utilising a persistent data structure.
However, since objects are being manoeuvred during consecutive tasks, roadmap-based SBPs need to re-evaluate all of their edge validity because the underlying \Cspace changes in-between each task.
The re-evaluation lowers the effectiveness of its multi-query property since it requires more time to build a sufficient roadmap when compared to other methods~\autocite{burget2016_BI2}.
On the other hand, existing tree-based SBPs have the properties of creating motion plans that inherently have lower cost and take less amount of time
\autocite{lai2018_BalaGlob}.
However, this class of methods needs to create a new motion plan from scratch for every new task.%
Since the manipulator always operates under the same environment (with objects moved by the manipulator itself), there are often many similar structure in-between each planning instance, for example, in consecutive pick-and-place tasks.

\begin{figure}[tb]
    \centering%
    \includegraphics[width=.4\linewidth]{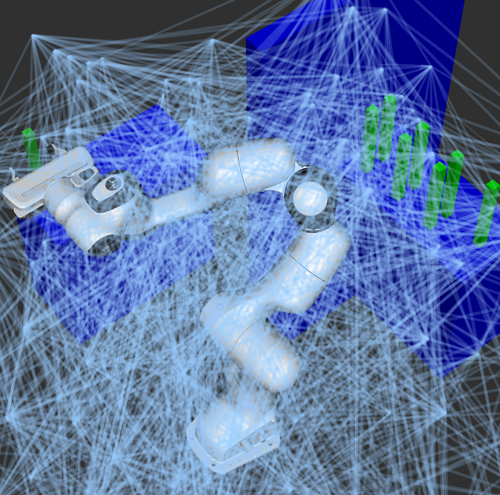}
    \includegraphics[width=.4\linewidth]{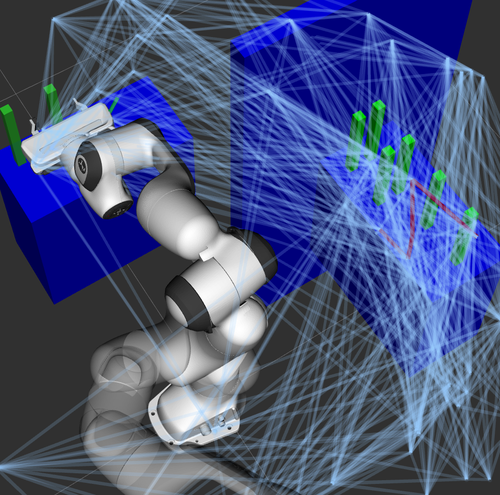}
    \caption{
        A PANDA arm is executing multiple pick-and-place tasks.
        The displayed connected graphs: (left) from a typical Lazy PRM*;
        (right) the experience graph from the proposed \LTR.
        Since the experience graph is built with a tree as a basis, it tends to inherit the tree-based planner property of concentrating edges closer to the initial and target locations.
        \label{fig:comp-resulting-graph}
    }
\end{figure}

The
pick-and-place problem differentiates itself from other planning problems by continuously operating in the same environment with spatial changes that are localised.
Localised changes refer to objects moved due to manipulation by the robot itself. These changes tend to be concentrated within certain regions and in the vicinity of the manipulator's gripper.
However, there is often valuable information collected from past planning problems that remain helpful in subsequent tasks.
Our proposed model remediates this issue by providing a framework that allows the planner to plans for optimal trajectories while utilising past connections previously known to be valid to bootstrap future solutions.

We propose a hybrid planner that tackles the consecutive tasks problem.
Unlike existing multi-query SBPs, our planner does not require any pre-processing graph construction step.
Instead, a rapid tree-based method is used to begin the planning process while simultaneously building a lazy experience graph to capture the connectivity.
There is minimal computational overhead involved since no collisions are performed yet.
In subsequent tasks, the lazy graph will then persist as a sparse structure that captures the previous experience under the same environment.
When the trees are close to the existing experience graph, the graph will be utilised to bootstrap the solution with the shortest path search.
Our planner has comparable performance against other tree-based methods in completely new environments.
However, in subsequent tasks, our planner utilises the lazy experience graph in a hybrid manner combining tree-based and graph-based approaches, often resulting in faster initial solutions and lower length cost.

\section{Related Works}
Sampling-based motion planners take a probabilistic approach to the planning problem.
There exist works that attempt to address the problem of being efficient in planning within the same environment.
For example, PRM~\autocite{kavraki1996_ProbRoad} performs planning within the same environment by saving previous requests in a graph-like structure.
The star-variants of SBPs often denotes their \emph{asymptotic optimality} nature in regards to minimising some cost functions.
Learning techniques like modelling the success probability~\autocite{lai2021rrf} or with a MCMC sampling scheme~\autocite{lai2020_BayeLoca} are possible approaches to improve performance.
Multiple tree or graph-like structures can also be utilised for efficient exploration.
For example, bidirectional trees~\autocite{kuffner2000_RRTcEffi} or multi-tree approach~\autocite{lai2018_BalaGlob} for object picking~\autocite{littlefield2016evaluating} to speed up the planning process by exploring different regions using a pre-computed database.
Combining tree-based and roadmap-based approaches had been explored in~\autocite{plaku2005sampling} which creates multiple trees in parallel but sacrifices solution speed due to its space-filling approach.

Since obstacles within the world-space would be altered during the course of the consecutive tasks problem, experience-based planners like PRM cannot directly reuse its roadmap for subsequent tasks.
Lazy-PRM*~\autocite{hauser2015_LazyColl} was designed to reduce the number of collision checking within its graph, but it can be altered to suit the need of the pick-and-place domain.
It lazily evaluates the validity of edges within the graph which saves computational time.
Elastic roadmaps~\autocite{yang2010elastic} had been used in mobile manipulators, which reuse manipulator planning in different 2D spatial locations. However, the approach does not benefit in fixed-frame manipulators.
E-graph~\autocite{chitta2013graphs} focuses on utilising heuristics to stores previous solution trajectories into a structure to accounts for environment changes. LPA$^*$~\autocite{koenig2004lifelong} maintain a structure after the first search, and reuses previous search tree that are identical to the new one.
The underlying changes in workspace can also be modelled as subset constrain manifold to aid solving subsequent tasks~\autocite{hauser2010multi}.

Consecutive tasks problem typically arises in Task and Motion Planning %
which can encapsulate the execution of tasks as skills~\autocite{lai2020_RobuHier} via learning state transitions~\autocite{lai2022_MEP}.
Techniques that focus on improve sampling efficiency can be modified to restrict the sampling region~\autocite{gammell2018_InfoSamp} or use a learned sampling distribution~\autocite{lai2020_LearPlan,ichter2018_LearSamp,sartoretti2019_PRIMPath,lai2021diffSamp,lai2021plannerFlows} to formulate a machine learning approach to learn from previous experience.
An experience-based database can help to improve planning efficiency.
For example, the
Lightning~\autocite{berenson2012_RoboPath} and Thunder~\autocite{coleman2015experience} framework uses some database structure to store past trajectories, but does not generalise solutions in-between trajectories.

\section{Lazy Tree-based Replanner}

In this section, we begin by formulating the consecutive tasks with objects manipulation problem under
optimal motion planning.
Then, we conceptualise the Lazy Tree-based Replanner framework and discuss the use of the Lazy Experience Graph.
Finally, we will analyse the algorithm tractability and completeness of the proposed algorithm.

\subsection{Formulation}

\newcommand{\task}{\tau}
\newcommand{\Task}{\mathrm{T}}
\newcommand{\loc}{\mathbf{x}}
\newcommand{\obj}{\Theta}

\begin{definition}[Object Manipulation Task]\label{def:pick-and-place}
    An object manipulation task $\task_i$ is a tuple $(\obj_{\task_i}, \loc_{\text{target},i})$ composed of an object $\obj_{\task_i}$ in the workspace $\omega$, and a target coordinate $\loc_{\text{target},i}\in\mathbb{SE}(3)$%
    .
    Let $\loc(\obj_{\task_i})\in\mathbb{SE}(3)$ denotes the current coordinate and orientation of the object $\obj_{\task_i}$.
    Then, an object manipulation task $\task_i$ refers to moving $\obj_{\task_i}$ from its initial pose $\loc(\obj_{\task_i})$ to its target pose $\loc_{\text{target},i}$ within the time budget.
\end{definition}

The repeated tasks scenario refers to a manipulator completing several instances of tasks (\cref{def:pick-and-place}) while avoiding collision with the environment.
Note that while the environment is static, the manipulator will modify the underlying \Cspace throughout the consecutive tasks due to the moved objects in $\omega$.
In the following, $\sigma:[0,1]\to C_\text{free}$ is a parameterised trajectory where $\sigma(0)$ and $\sigma(1)$ refers to the start and end of the trajectory in the configuration space.

\begin{problem}[Consecutive Manipulation Tasks Problem]\label{problem:pap-task}
    Given the set of configuration space $C \subseteq \mathbb{R}^d$, the set of free space $C_\text{free} \subseteq C$, a set of objects manipulation tasks $\Task = \{\task_1, \ldots, \task_n\}$, and a function $f_q\colon \omega \to C_\text{free}$ that maps any given object $\obj_{\task_i}$ from its current pose to a reachable and feasible grasp configuration.
    Formally, find a trajectory $\sigma$ such that all $\task_i\in\Task$ are completed according to~\cref{def:pick-and-place} using $f_q$ for grasping each $\obj_{\task_i}$, where $\forall\, \theta \in [0, 1], \sigma(\theta) \in C_\text{free}$.
\end{problem}

\begin{problem}[Optimal Planning]\label{problem:optimal-task}
    Let $\Gamma(C_\text{free})$ denotes the set of all possible trajectories in $C_\text{free}$.
    Given $C$, $C_\text{free}$, $\Task$ with specified task execution order, and a cost function $\mathcal{L}: \sigma \to [0, \infty)$;
    find a solution trajectory $\sigma^*$ that incurs the lowest cost.
    Formally, find trajectory $\sigma^*$ such that
    \begin{equation}
        \mathcal{L}(\sigma^*) = \min_{ \sigma \in \Gamma(C_\text{free})} \mathcal{L}_c(\sigma)
    \end{equation}
    and it satisfies~\cref{problem:pap-task}.
\end{problem}

\subsection{Lazy Tree-based Replanner}

We propose Lazy Tree-based Replanner (\LTR)---an incremental, asymptotic optimal sampling-based planner capable of rapidly replans for executing consecutive tasks.
\LTR unifies tree-based and roadmap-based methods by utilising the rapid sampling procedure with bidirectional trees whilst using a \emph{lazy experience graph} to captures the connectivity in \Cspace for future tasks.
Our target domain---consecutive object manipulation planning---is inherently a multi-query problem.
However, unlike typical multi-query planners, \LTR avoids pre-processing overhead by incrementally constructs its lazy-experience graph using a single-query motion planner.
When \LTR is first planning under a new environment, the procedure is similar to that of a single-query SBP.
\LTR first utilises bidirectional trees~\autocite{klemm2015_RRTCFast} to creates an initial motion plan, which is a state-of-the-art asymptotic optimal planner that often outperforms other tree-based planners.
During the planning procedure, \LTR captures the connectivity of \Cspace with a lazy experience graph similar to that of a \LazyPRM.
The graph is for capturing the connectivity in the current environment, and the validity of the edges will not be evaluated yet.
In subsequent tasks, \LTR outperforms other methods by bootstrapping solution with its experience graph.
Rapid tree-based methods are used to expands outward from its initial and target configurations.
Then, when the tree structure is close to its previous lazy experience graph, it will perform a lazy shortest path search within the graph to connect both trees.
Since the lazy graph captures \Cspace connectivity from its previous tasks, most connections will be valid even if there is a modification in the \Cspace due to the moved objects.
As a result, \LTR can rapidly replan for subsequent tasks, e.g. in consecutive pick-and-place tasks, by utilising its experience graph.

\begin{algorithm}[bt]%
    \caption{\LTR in Consecutive Pick-and-Place tasks} \label{alg:pap-task}
    \KwIn{$\Task, {\epsilon_n}, \delta_t$ (time budget per task)}

    \Fn{$\FnPlanTrajectory(\mathcal{G}, q_\text{init}, q_\text{target})$}{
        $\mathcal{T}_\text{initial}, \mathcal{T}_\text{target} \gets $ initialise trees with $q_\text{init}$ and $q_\text{target}$ \;
        $\mathcal{G}'\gets (V=\{\; q_\text{init}, q_\text{target} \;\}, E=\emptyset)$;~$\sigma\gets\emptyset$ \;
        \While{within time budget $\delta_t$}{
            \For{$\mathcal{T} \in \{ \mathcal{T}_\text{initial}, \mathcal{T}_\text{target} \}$}{
                $v_\text{new}, e_\text{new} \gets \texttt{GrowTree}(\mathcal{T}, {\epsilon_n}, \mathcal{T}_\text{initial}, \mathcal{T}_\text{target})$ \label{alg:pap-task:grow-tree}\;
                $\mathcal{G}' \gets$ add $e_\text{new}$ as validated edge; add $v_\text{new}$ and connect to all $\{v'\in\mathcal{G}' \given \dist{v' - v_\text{new}} \le {\epsilon_n}\}$ \label{alg:pap-task:saves-to-G'}\;
                \If{$\mathcal{T}_\text{initial}$ connects to $\mathcal{T}_\text{target}$}
                {$\sigma \gets$ path between $\mathcal{T}_\text{initial}$ and $\mathcal{T}_\text{target}$ \;}
                \For{$v\in\mathcal{G}$ {\normalfont{\textbf{where}}} $\dist{v - v_\text{new}} \le {\epsilon_n}$ \label{alg:pap-task:for-v-in-G:start}}{
                    \If{can connect from $v_\text{new}$ to $v$}{
                        \Comment{Found a path from $\mathcal{T}$ to $\mathcal{G}$}
                        $\mathcal{T} \gets$ add $v$ to tree and rewires \label{alg:pap-task:for-v-in-G:end} \;
                    }
                }
                \If{$\sigma$ is empty and $\mathcal{T}_\text{initial}$, $\mathcal{T}_\text{target}$ connects $\mathcal{G}$ \label{alg:pap-task:if-T-T-connected:start}}{
                    $\sigma \gets$ get solution with shortest path search \;
                    \label{alg:pap-task:if-T-T-connected:end}\Comment{can return here for fast sol.}
                }
            }
        }
        $\mathcal{G} \gets$ merge $\mathcal{G}$ with $\mathcal{G}'$ \;
        \Return $\sigma$ \;
    }

    $\mathcal{G}\gets (V=\emptyset, E=\emptyset)$  \label{alg:pap-task:main-start}\Comment*[r]{lazy experience graph}
    \For{task $\task_i \in \Task$ \label{alg:pap-task:loop-start}}{
        $q_\text{current} \gets$ current manipulator configuration \;
        $\FnPlanTrajectory(\mathcal{G}, q_\text{current}, f_q(\loc(\obj_{\task_i})))$ \;
        pickup object $\obj_{\task_i}$ \;
        $\FnPlanTrajectory(\mathcal{G}, f_q(\loc(\task_i)), f_q(\loc_{\text{target},i}))$ \;
        place object $\obj_{\task_i}$ \label{alg:pap-task:loop-end}\;
    }
\end{algorithm}

\begin{figure*}[t!]
    \centering%
    \includegraphics[width=.35\linewidth]{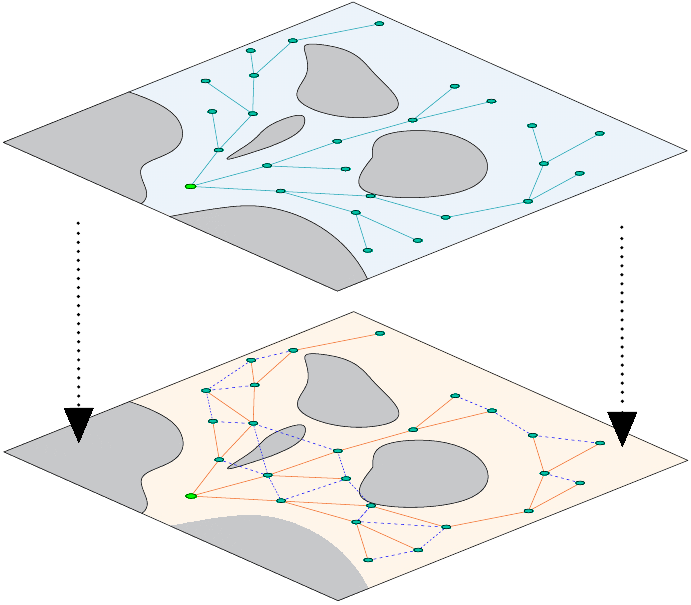}
    \hfill
    \includegraphics[width=.6\linewidth]{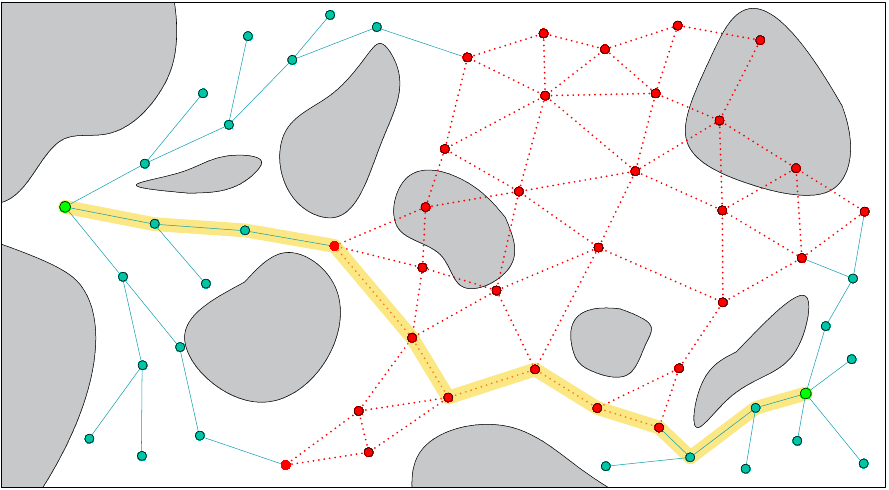}
    \caption{
        An overview of the \LTR algorithm. Grey blobs, circle markers and line segments represent obstacle space, nodes and edges, respectively.
        (\textbf{Left})
            When \LTR grows its trees (top layer) from its initial configuration (green node), it simultaneously builds a lazy experience graph (bottom layer).
            Solid orange edges in the bottom layer retain the checked validity from its original tree structure, whereas the blue dotted edges are added without checking validity yet (hence there are some invalid connections in the example given).
        (\textbf{Right})
            Illustration of \LTR using its previously built lazy experience graph (red graph) to speed up the solution.
            The two cyan trees on either side are built for the current pick-and-place task, and when those trees are in close proximity to the experience graph, it can return a fast solution by performing a shortest path graph-search at the connected nodes (resulting in yellow solution trajectory).
            Note that the simultaneously built experience graphs for this task are not shown (i.e. the orange graph at the bottom left layer) for clarity.
            The red (past experience) and orange (newly built experience) graphs will be merged for subsequent tasks.
        \label{fig:factorio-environment-subgraph}
    }
\end{figure*}

\subsection{Implementation}
\Cref{alg:pap-task} illustrates \LTR's overall algorithmic aspect.
\Cref{alg:pap-task:main-start} begins the consecutive pick-and-place tasks problem by first initialising its persistent lazy experience graph $\mathcal{G}$.
Incoming tasks are processed in lines \ref{alg:pap-task:loop-start} to \ref{alg:pap-task:loop-end}, where the execution order of tasks are not of the scope of this problem.
Each planning request is sent to the $\FnPlanTrajectory$ subroutine, where $f_q\colon \omega \to C_\text{free}$ is a function that maps from object to a possible grasp configuration.
Notice that the underlying \Cspace structure changes (i.e. changing collision geometry) after each \emph{pickup} and \emph{place} operation due to the object being attached to the manipulator and moved objects in $\omega$.

$\FnPlanTrajectory$ subroutine takes two configurations as its planning input, along with the persistent graph that saves its prior experience.
When $\mathcal{G}$ is empty, \LTR behaves similar to that of a bidirectional RRT* as lines~\ref{alg:pap-task:for-v-in-G:start} to~\ref{alg:pap-task:if-T-T-connected:end} would not be entered.
While \LTR is growing its trees, it also saves the created vertices and edges information in a temporary graph $\mathcal{G}'$ (\cref{alg:pap-task:saves-to-G'}) which would be merged with $\mathcal{G}$ at the end of a task.
There are no additional computation overheads when building $\mathcal{G}'$ except that of the nearest neighbour search (NNS) in a k-d tree structure since collision checking is not performed yet (Left of \cref{fig:factorio-environment-subgraph}).
The search can be further optimised out by utilising information from $\texttt{GrowTree}$ (\cref{alg:pap-task:grow-tree}) since a typical tree-based planner also uses NNS, and both planners operate on the same set of vertices.
Similar to existing bidirectional tree approaches (i.e., Alg. 1 in ~\autocite{klemm2015_RRTCFast}), \texttt{GrowTree} subroutine
samples a random configuration and expands the current tree $\mathcal{T}$ while attempting to connects towards the other tree.
When the NNS found some $v_\text{new}$ within $\epsilon_n$ of $\mathcal{G}$ during the tree expansion, \LTR will attempts to connect $\mathcal{T}$ with $\mathcal{G}$ (\crefrange{alg:pap-task:for-v-in-G:start}{alg:pap-task:for-v-in-G:end}).
When both trees are connected to $\mathcal{G}$, \LTR can bootstrap a solution by performing a shortest path search in $\mathcal{G}$ at the two connected vertices.
Edges will be lazily validated in the typical Lazy PRM* fashion and saved for subsequent search~\autocite{hauser2015_LazyColl}.
If the shortest path returns a valid solution at the two connected vertices, then \LTR has successfully found a solution that connects from the initial configuration, through $\mathcal{T}_\text{initial}$, $\mathcal{G}$ and $\mathcal{T}_\text{target}$, to the target configuration (\cref{alg:pap-task:if-T-T-connected:start,alg:pap-task:if-T-T-connected:end}).
In~\cref{fig:factorio-environment-subgraph}, the right illustration provides an overview of the above scenario.

\LTR follows the approach in~\autocite{karaman2011_SampAlgo} which uses the connection radius $\epsilon_n$ as a function of the number of node $n\in\mathbb{N}$, where $\epsilon_n = \epsilon(n) := \gamma (\log{n}/n)^{1/d}$, $\gamma > \gamma^* = 2(1+1/d)^{1/d}(\mu(C_\text{free})/\xi_d)^{1/d}$, $\xi_d$ is the volume of the unit ball in $d$-dimensional Euclidean space, and $\mu(X)$ denotes the Lebesgue measure of a set $X$.
In practice, the lazy experience graph built by \LTR tends to inherit the tree-based planners' property of concentrating more edges closer to the initial and target configurations (\cref{fig:comp-resulting-graph}).
While such a graph is not as diverse as that of a \LazyPRM, it is, in fact, more beneficial under a wide range of manipulator and environment setups since most pickup and place tasks will be concentrated at some central location.
As a result, the lazy experience graph helps to bootstrap solutions that often exhibit lower costs than \LazyPRM.

\subsection{Algorithmic tractability}\label{sec:analysis}

\LTR uses experience graph from past plannings to explores \Cspace, which is guaranteed to help the current planning if its prior knowledge is viable in the current settings.
Furthermore, \LTR is probabilistic complete and will asymptotically converge its solution to the optimal one.
In the following, we will use $V_d(R)$ to denotes the $d$-dimensional volume of a hypersphere of radius $R$, and $\mathcal{V}(q)$ to denote the visibility set of $q$ which represents the region of $C_\text{free}$ visible from some $q\in C_\text{free}$.

\begin{definition}[Connected free space]\label{def:connected-freespace}
    $C'_\text{free} \subseteq C_\text{free}$ is
    said to be
    a connected subset of free space if $\forall m \in \mathbb{N}, m > 2, \exists \varepsilon  > \epsilon_n$ for $n \ge m$ such that
    \begin{equation}
        \label{eq:visibility_in_C'free}
        \forall q \in C'_\text{free},~\mu(\mathcal{V}(q)) \ge V_d({\varepsilon}).
    \end{equation}
\end{definition}
\vspace{2mm}
\Cref{def:connected-freespace} states that each configuration in $C'_\text{free}$ contains at least ${\varepsilon}$ radius of surrounding free space which will eventually surpass the shrinking radius $\epsilon_n$ when $n \ge m$, where $m$ is environment specific.

\begin{theorem}[Joining of Lazy Experience Graph] \label{thm:join-lazy-experience-graph}
  Let free space $C_\text{free} \subseteq C$ be a bounded set.
  Consider $C'_\text{free}\subseteq C_\text{free}$ to be a connected subset of the free space $C_\text{free}$ according to~\cref{def:connected-freespace}.
  Given that $\mathcal{G}\neq \emptyset$.
  If $\mathcal{T}_\text{initial}, \mathcal{T}_\text{target}, \mathcal{G} \subseteq C'_\text{free}$,
  $\mathcal{T}_\text{initial}$ connects $\mathcal{G}$ and $\mathcal{T}_\text{target}$ connects $\mathcal{G}$ with probability one.

\end{theorem}
\begin{proof}
    From \cref{alg:pap-task} lines~\ref{alg:pap-task:for-v-in-G:start} to~\ref{alg:pap-task:if-T-T-connected:end}, we add each new node from $\mathcal{T}_\text{initial}$ that are reachable and within ${\epsilon_n}$ radius to $\mathcal{G}$.
    The event of joining the tree $\mathcal{T}_\text{initial}$, and $\mathcal{G}$ occurs if a newly sampled configuration $q_\text{new}$ lies within a region that is visible and connectable to both $\mathcal{T}_\text{initial}$ and $\mathcal{G}$.
    The connectable region $S$ is given by
    \begin{equation}
        S=
        \bigcup_{q \in \mathcal{T}_\text{initial}} \mathcal{V}_{\epsilon_n}(q)
        \cap \bigcup_{q \in \mathcal{G}} \mathcal{V}_{\epsilon_n}(q)
    \end{equation}
    where
    \begin{equation}
        \mathcal{V}_{R}(q):=\set{q_i \in \mathcal{V}(q) \given \norm{q_i-q} \le R}
    \end{equation}
    denotes the visibility set of $q$ bounded by the hypersphere of radius $R$.
    The entire experience graph and the tree will occupies
    $\volof{\mathcal{G}}$ and $\mu^{\mathcal{T}_\text{initial}}$ amount of volume in $C'_\text{free}$ where
    \begin{equation}
        \volof{\mathcal{G}} = \bigcup_{q \in \mathcal{G}} \mu(\mathcal{V}_{\epsilon_n}(q))
    ~~~~\text{and}~~~~
        \volof{\mathcal{T}_\text{initial}} = \bigcup_{q \in \mathcal{T}_\text{initial}} \mu(\mathcal{V}_{\epsilon_n}(q)).
    \end{equation}
    Eq. \Cref{eq:visibility_in_C'free} states that each $q$ has at least $V_d({\varepsilon}) > V_d({\epsilon_n})$ free space to be connected with surrounding configurations with ${\epsilon_n}$ connection radius.
    If the lazy experience graph $\mathcal{G}$ is non-empty (from past planning instances), \cref{eq:visibility_in_C'free} further ensures that the connectable space from $\volof{\mathcal{G}}$ and $\volof{\mathcal{T}_\text{initial}}$ has non-zero volume, and hence non-zero probability of begin sampled by a uniform sampler.
    Since the tree $\mathcal{T}_\text{initial}$ is growing in $C'_\text{free}$, the tree will eventually reach all visible configurations in $C'_\text{free}$ as the number of nodes approach infinity (see~\autocite{klemm2015_RRTCFast}).
    Therefore, the probabily that $\mathcal{T}_\text{initial}$ joins with $\mathcal{G}$ is given by
    \begin{equation}
        \lim_{n\to\infty}
        \mathbb{P}(
            q_n \in S
        ) = 1,
    \end{equation}
    where $q_n$ is the $n$\textsuperscript{th} sample of configuration for $\mathcal{T}_\text{initial}$.
    The same proof can be applied to the joining of $\mathcal{T}_\text{target}$ and $\mathcal{G}$.
\end{proof}

\begin{table*}[t!]
    \begin{minipage}{0.75\linewidth}
    \centering
        \caption{
            Numerical results for the Panda, Jaco and TX90 arm scenarios ($\mu \pm \sigma$).
            \label{table:numerical-results}
        }
        \resizebox{\linewidth}{!}{%
        \begin{tabular}{@{}cc@{}ccccccc@{}}
        \toprule
         &  &  & \multicolumn{2}{c}{\textbf{Panda arm}} & \multicolumn{2}{c}{\textbf{Jaco arm}} & \multicolumn{2}{c}{\textbf{TX90 arm}} \\ \cmidrule(l){4-9}
         &  \multicolumn{2}{c}{\# Task}  & \multicolumn{1}{c}{Time (s)} & \multicolumn{1}{c}{Cost} & \multicolumn{1}{c}{Time (s)} & \multicolumn{1}{c}{Cost} & \multicolumn{1}{c}{Time (s)} & \multicolumn{1}{c}{Cost} \\ \midrule
        \multirow{4}{*}{LTR*} & \multirow{2}{*}{Pick} & \multicolumn{1}{c|}{$1^\text{st}$} & 2.36 ± 1.42 & 9.13 ± 3.76 & 12.75 ± 6.73 & 8.72 ± 3.85 & 11.73 ± 8.91 & 14.91 ± 31.77 \\
         &  & \multicolumn{1}{c|}{$8^\text{th}$} & 0.38 ± 0.25 & 10.17 ± 4.16 & 3.17 ± 3.75 & 6.92 ± 2.76 & 5.12 ± 3.73 & 9.94 ± 10.70 \\ \cmidrule(l){2-9}
         & \multirow{2}{*}{Place} & \multicolumn{1}{c|}{$1^\text{st}$} & 1.23 ± 1.14 & 8.45 ± 2.97 & 15.16 ± 8.63 & 9.15 ± 4.92 & 12.76 ± 7.21 & 13.64 ± 28.78 \\
         &  & \multicolumn{1}{c|}{$8^\text{th}$} & 0.18 ± 0.16 & 9.91 ± 0.58 & 4.75 ± 4.76 & 7.31 ± 3.75 & 4.9 ± 4.91 & 15.17 ± 8.53 \\ \midrule
        \multirow{4}{*}{\LazyPRM} & \multirow{2}{*}{Pick} & \multicolumn{1}{c|}{$1^\text{st}$} & 0.91 ± 1.26 & 22.91 ± 15.82 & 9.28 ± 7.26 & 23.72 ± 18.87 & 10.67 ± 10.75 & 51.26 ± 35.73 \\
         &  & \multicolumn{1}{c|}{$8^\text{th}$} & 1.31 ± 0.36 & 23.86 ± 23.83 & 2.86 ± 6.82 & 17.82 ± 14.47 & 3.76 ± 7.81 & 29.77 ± 19.67 \\ \cmidrule(l){2-9}
         & \multirow{2}{*}{Place} & \multicolumn{1}{c|}{$1^\text{st}$} & 1.83 ± 2.58 & 24.73 ± 22.18 & 8.96 ± 6.17 & 22.27 ± 20.21 & 9.93 ± 11.38 & 48.75 ± 46.73 \\
         &  & \multicolumn{1}{c|}{$8^\text{th}$} & 0.85 ± 0.41 & 20.83 ± 17.75 & 3.83 ± 8.91 & 19.51 ± 12.79 & 4.78 ± 8.37 & 22.29 ± 27.76 \\ \midrule
        \multirow{4}{*}{\RRTC} & \multirow{2}{*}{Pick} & \multicolumn{1}{c|}{$1^\text{st}$} & 1.47 ± 2.49 & 9.86 ± 4.78 & 8.47 ± 7.74 & 7.1 ± 13.85 & 8.18 ± 9.74 & 14.86 ± 6.09 \\
         &  & \multicolumn{1}{c|}{$8^\text{th}$} & 2.11 ± 2.25 & 10.18 ± 4.17 & 9.17 ± 10.07 & 8.11 ± 10.33 & 8.74 ± 9.13 & 13.73 ± 7.83 \\ \cmidrule(l){2-9}
         & \multicolumn{1}{l}{\multirow{2}{*}{Place}} & \multicolumn{1}{c|}{$1^\text{st}$} & 1.32 ± 0.98 & 8.91 ± 4.16 & 11.37 ± 8.74 & 9.84 ± 9.74 & 7.99 ± 7.99 & 10.70 ± 6.74 \\
         &  & \multicolumn{1}{c|}{$8^\text{th}$} & 1.28 ± 2.1 & 9.17 ± 5.83 & 7.36 ± 9.36 & 8.76 ± 12.57 & 9.36 ± 8.67 & 12.85 ± 11.76 \\ \midrule
        \multirow{4}{*}{E-Graph} & \multirow{2}{*}{Pick} & \multicolumn{1}{c|}{$1^\text{st}$} & 2.76 ± 1.01 & 17.80 ± 10.27 & 13.58 ± 6.85 & 28.28 ± 11.58 & 13.86 ± 5.86 & 44.18 ± 29.18 \\
         &  & \multicolumn{1}{c|}{$8^\text{th}$} & 1.77 ± 2.92 & 15.91 ± 16.09 & 7.58 ± 11.53 & 23.46 ± 16.95 & 15.72 ± 8.19 & 55.28 ± 30.98 \\ \cmidrule(l){2-9}
         & \multicolumn{1}{l}{\multirow{2}{*}{Place}} & \multicolumn{1}{c|}{$1^\text{st}$} & 2.69 ± 1.38 & 15.26 ± 11.58 & 15.18 ± 7.91 & 28.66 ± 15.85 & 10.99 ± 11.27 & 39.70 ± 26.87 \\
         &  & \multicolumn{1}{c|}{$8^\text{th}$} & 1.82 ± 3.27 & 12.39 ± 8.11 & 9.91 ± 11.28 & 19.68 ± 13.12 & 9.36 ± 8.67 & 43.58 ± 25.33 \\ \bottomrule
        \end{tabular}
        }
    \end{minipage}\hfill%
    \begin{minipage}{0.235\linewidth}
        \centering
        \includegraphics[width=\linewidth]{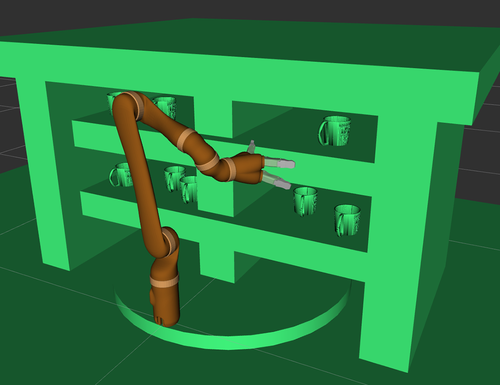}\vspace{1mm}
        \includegraphics[width=\linewidth]{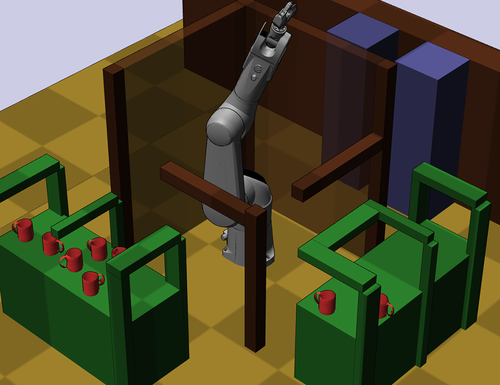}
        \captionof{figure}{
            (Top) Jaco and (bottom) TX90 arm scenarios.
            \label{fig:jaco-and-tx90}
        }
    \end{minipage}
\end{table*}

\LTR attains probabilistic completeness as its number of uniform random samples approaches infinity.

\begin{theorem}[Probabilistic Completeness] \label{thm:probabilistic-complete}
    \LTR inherits the same probabilistic completeness of bidirectional \RRT in each pick-and-place task.
    Let $q_{\text{initial}}^{\task_i}$ and $q_{\text{target}}^{\task_i}$ denotes the initial and target configuration provided in each task $\task_i \in \Task$, and $V_{t}^{\task_i}$ denotes the set of configurations reachable from $q_{\text{initial}}^{\task_i}$  through connected edges for $\task_i$.
    Given that a solution exists for each pair $q_{\text{initial}}^{\task_i}$ and $q_{\text{target}}^{\task_i}$, \LTR is guaranteed to eventually find a solution.
    Formally,
    \begin{equation}
        \forall \task_i \in \Task, \lim_{t\to\infty}\mathbb{P}(
            \{V_{t}^{\task_i}
                \cap
            \{q_{\text{target}}^{\task_i}\}
        \ne\emptyset\})=1.
    \end{equation}
\end{theorem}
\vspace{2mm}
\begin{proof}
    If a solution exists for task $\task_i$, then $q_{\text{initial}}^{\task_i}$ and $q_{\text{target}}^{\task_i}$ must lies in the same connected free space $C'_\text{free}$.
    If part of $\mathcal{G}$ also lies in the the same $C'_\text{free}$, $q_{\text{initial}}^{\task_i}$ and $q_{\text{target}}^{\task_i}$ are guaranteed to connects to $\mathcal{G}$ (\cref{thm:join-lazy-experience-graph}).
    In the event that the existing lazy experience graph does not lies in $C'_\text{free}$ (e.g. first pick-and-place task or substantially modified workspace), \LTR behaves the same as that of a bidirectional \RRT because no connections would be made to $\mathcal{G}$ and \LTR would be reduced to its internal planner's subroutine
    (\cref{alg:pap-task} \cref{alg:pap-task:for-v-in-G:start} to \cref{alg:pap-task:if-T-T-connected:end} would not be entered).
    Therefore, \LTR attains probabilistic completeness and its asymptotic behaviour is same as that of a bidirectional \RRT.
\end{proof}

\LTR uses an experience graph to bootstraps fast solution trajectory; hence, unlike \RRT variants, \LTR does not possess the anytime property of all tree branches always exhibit the current optimal route.
However, \LTR inherits asymptotic optimality where the solution trajectory will converges to the optimal trajectory almost surely.

\begin{theorem}\label{thm:asym-same-as-rrt}
    Let $n^{LTR*}_i$ and $n^{bRRT*}_i$ be the number of uniformly random configurations sampled at iteration $i$ for \LTR and bidirectional RRT* respectively.
    There exists a constant $\phi \in \mathbb{R}$ such that
    \begin{equation}
        \lim_{i\to\infty} \mathbb{E}[ n^{LTR*}_i / n^{bRRT*}_i ] \le \phi.
    \end{equation}
\end{theorem}
\vspace{2mm}
\begin{proof}
    As stated in \cref{thm:join-lazy-experience-graph}, $\mathcal{T}_\text{initial}$ and $\mathcal{T}_\text{target}$ has non-zero probability of adding new nodes at each iteration, hence as $i\to\infty$ both trees will have infinite samples to refine its branches.
    In the \LTR algorithm, it uses the experience graph $\mathcal{G}$ to bootstraps solution.
    However, the internal structure of the trees are maintained separately (\cref{alg:pap-task}~\cref{alg:pap-task:grow-tree}), and edges within $\mathcal{G}$ are not added to the trees when using $\mathcal{G}$ to bridge the gap between $\mathcal{T}_\text{initial}$ and $\mathcal{T}_\text{target}$.
    Therefore, the asymptotic behaviour of the growth of $\mathcal{T}_\text{initial}$ and $\mathcal{T}_\text{target}$ in \LTR approaches bidirectional RRT* as $i\to\infty$.
\end{proof}

Although \LTR uses the (possibly non asymptotic optimal compliant) $\mathcal{G}$ to bootstraps the connection between $\mathcal{T}_\text{initial}$ and $\mathcal{T}_\text{target}$, \LTR guarantees the solution it returns will asymptotically converges to the optimal solution.

\begin{theorem}[Asymptotic optimality] \label{thm:optimality}
    Let $\sigma^{*}_i$ be \LTR's solution at iteration $i$, and $c^*$ be the minimal cost for \cref{problem:optimal-task}.
    If a solution exists, then the cost of $\sigma_i$ will converge to the optimal cost almost-surely.
    That is,
    \begin{equation}
        \mathbb{P} \left( \left\{ \lim_{i\to\infty} c(\sigma_i) = c^* \right\} \right) = 1.
    \end{equation}
\end{theorem}
\vspace{2mm}
\begin{proof}
    Since \LTR uses experience graph that are not optimised against the current task's $q_\text{initial}$ and $q_\text{initial}$, the solution trajectory $\sigma_i$ that \LTR returns at iteration $i$ does not satisfy the asymptotic criteria specified in~\autocite{karaman2011_SampAlgo}.
    However, even after $\mathcal{T}_\text{initial}$ and $\mathcal{T}_\text{target}$ are joined with $\mathcal{G}$ (i.e. after obtaining fast initial solution), \LTR only adds $v\in\mathcal{G}$ within ${\epsilon_n}$ to $\mathcal{T}$ (\cref{alg:pap-task}~\cref{alg:pap-task:for-v-in-G:start}) sequentially and with adequate rewire procedure (\cref{alg:pap-task}~\cref{alg:pap-task:for-v-in-G:end}) within the internal tree structure.
    Therefore, as the asymptotic behaviour of \LTR approaches that of a bidirectional \RRT (\cref{thm:asym-same-as-rrt}) after $\mathcal{T}_\text{initial}$ and $\mathcal{T}_\text{target}$ connects without the bridging of $\mathcal{G}$, the cost of $\sigma_i$ that \LTR returns converges to the optimal cost almostly surely.
\end{proof}

\section{Experiments}
We experimentally evaluate the versatility and performance of \LTR by testing for speed and cost in multiple environments.
Since different types of manipulators exhibit distinct \Cspace, which will impact the underlying structures, we conduct our experiment on three different types of manipulators to assess \LTR's robustness.
The simulated environments are depicted in~\cref{fig:comp-resulting-graph,fig:jaco-and-tx90}.
The manipulators used are the PANDA, Kinova Jaco and TX90 arm with an attached PR2 gripper.
We first describe the environments and experimental setup, then provide a comparison against other state-of-the-art sampling-based planners under the same environment.

\begin{figure*}[t!]
    \centering%
    \includegraphics[width=.9\linewidth]{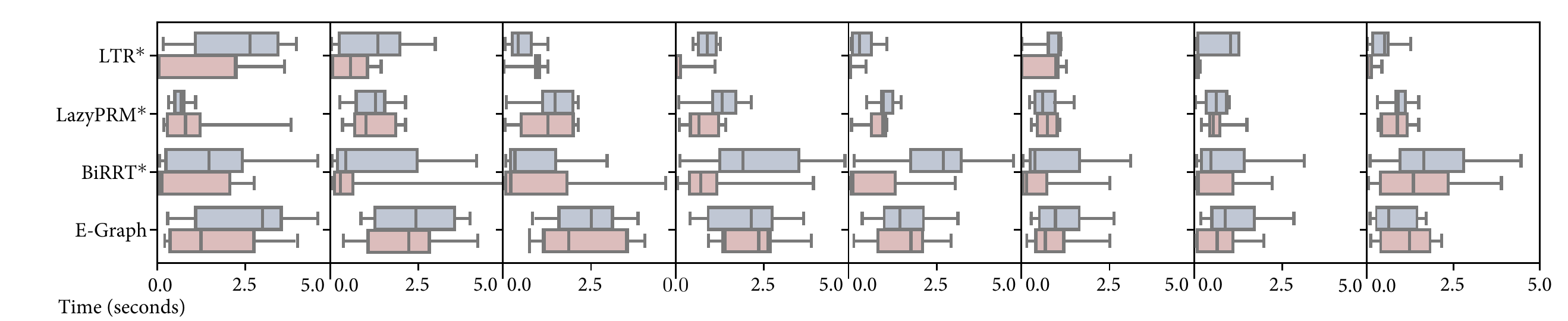}
    \includegraphics[width=.9\linewidth]{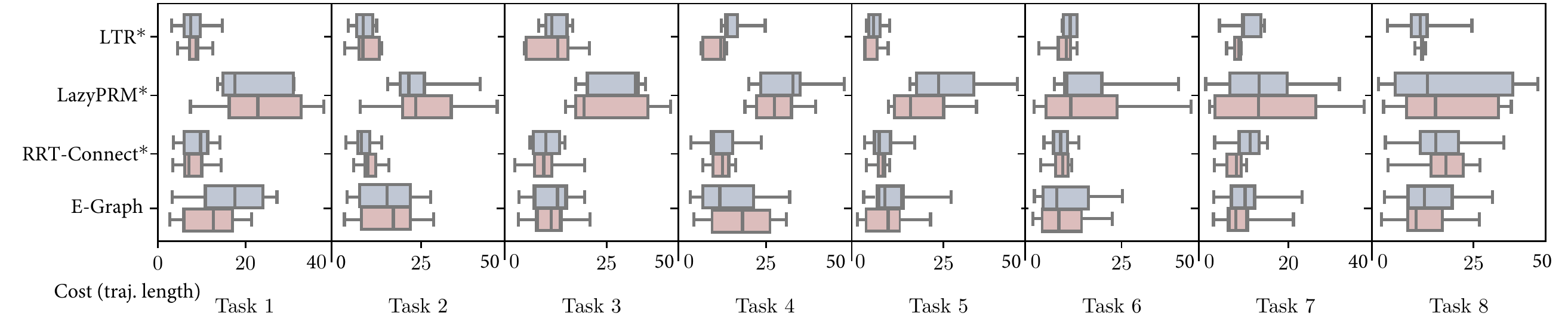}
    \caption{
        Experimental results obtained from the PANDA consecutive pick-and-place tasks, to illustrate the evolution of \LTR's performance after perfroming multiple tasks.
        The blue and red boxes correspond to the pickup and place task, respectively.
        Each planner has to plan a trajectory to pickup the object and place at the target location to complete each task.
        Boxplots of
        (\textbf{Top})
             time to obtain initial solution and
        (\textbf{Bottom})
            the final cost (length) of the trajectory.
        \label{fig:exp}
    }
\end{figure*}

The scenario's objective is to manoeuvre all of the objects $\obj_{\task_i}$ from their randomised starting location to their target location $\loc_{\text{target},i}$.
The starting and target regions are randomly generated on some predefined surface for each scenario.
There exist eight objects in each scenario, often on the surface of tables or cupboards. The objective of each task is to transfer an object from a starting location to a target location while avoiding collisions.
We implement
\LTR under the OMPL
framework, and integrates it with the MoveIt and Klampt simulator in Python \cite{lai2021SbpEnv}.
The experiment is composed of 8 tasks, each of which requires a planning procedure for picking up and placing an object at some pre-generated random location.
The manipulator will need to go back and forth through a constantly changing region, resulting in 16 motion trajectory execution.
\LTR, \RRTC~\autocite{klemm2015_RRTCFast}, \LazyPRM~\autocite{hauser2015_LazyColl} and E-Graph~\autocite{chitta2013graphs} are tested in this experiment.
The experiment for each planner is repeated 10 times each, as a total of 160 trajectories for each scenario.
The lazy data structure of each planner is kept between each task, except bidirectional RRT* since it does not allow invalidating edges.
The data structure of each planner is initially empty, except for E-Graph, where we first bootstrap its data structure with 5 uniformly distributed goals.
\begin{figure}[tb]
    \centering
    \includegraphics[width=.95\linewidth]{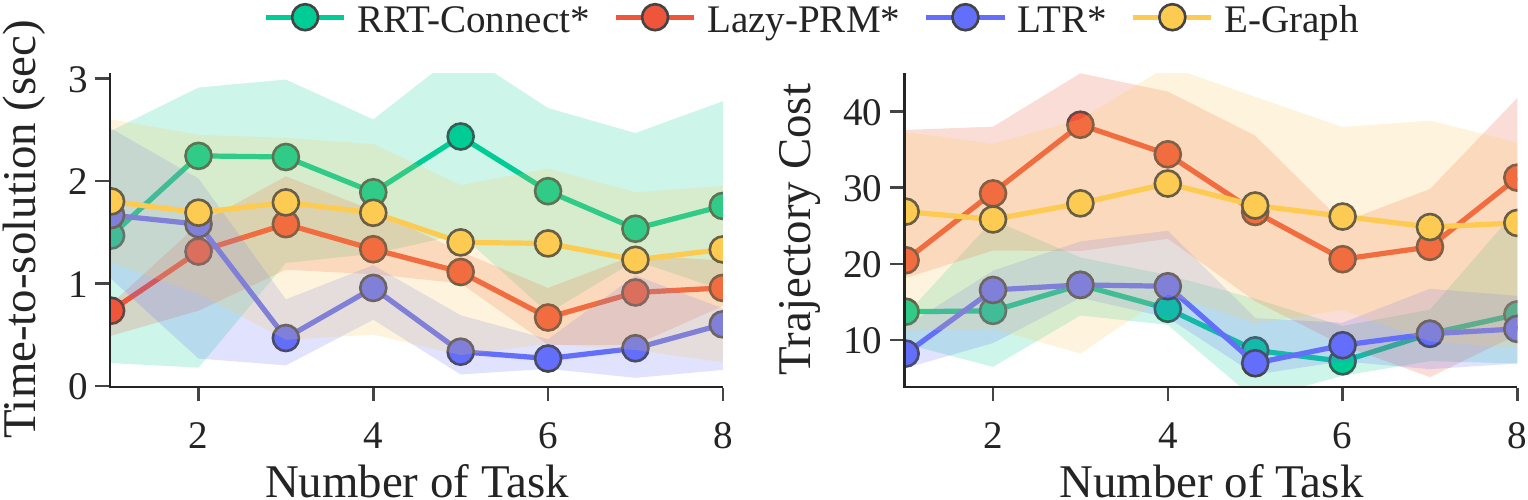}
    \caption{Experimental results for the multiple object picking tasks in PANDA problem setup
    (line and shaded region depict median $\pm 25\%$ quartiles).
    \label{fig:panda-line-plot}
    }
\end{figure}
Experimental results obtained is shown in~\cref{table:numerical-results}.
Results from the 1\textsuperscript{st} and 8\textsuperscript{th} tasks are included to highlight the performance gain by exploiting the consecutive nature of the problem.
\LTR and \LazyPRM operate at a similar speed on their 1\textsuperscript{st} task and consistently improve their \emph{time} metric (time required to obtain first solution) by the 8\textsuperscript{th} task via reusing their stored graph; whereas \RRTC needs to recompute its entire tree from scratch and does not improve its speed.
E-Graph requires an initial bootstrap as it uses 3D Dijkstra and inverse kinematics to facilitate planning, which in average requires more than 30 seconds to obtain an initial solution in the PANDA environment with an empty graph. E-Graph gradually improves its time-to-solution as it stores more solution, but with a degraded solution cost. Moreover, E-Graph cannot provides any completeness and optimal convergence guarantee as it operates in workspace.
Solutions returned by \RRTC always consists of lower cost due to its tree-based nature of expanding configurations directly from the initial and target configurations, which is also true for \LTR since \LTR also bootstraps its lazy experience graph with trees. As a result, \LTR behaves like a hybrid in-between the other two, inheriting the benefit of graph reusability and low-cost solution nature.
In~\cref{fig:exp}, we visually display results from the PANDA arm as boxplots of time and cost across the evolutions of tasks 1 to 8.
The x-axis is divided into eight sections; each corresponds to picking up an object and placing it at the dedicated location.
\LTR behaves similar to \RRTC on the first task when the lazy experience graph has not been built yet (top of~\cref{fig:exp}).
This is because the subroutine as depicted in~\cref{alg:pap-task} will not be entered on the first task (empty graph $\mathcal{G}$).
However, in subsequent tasks, \LTR can return a much faster solution by utilising the lazy experience graph to bootstrap an initial solution.
The result is similar to that of a \LazyPRM where it can reuse its persistent data structure in subsequent tasks.
\Cref{fig:panda-line-plot} depicts the same phenomenon with x-axis indicating \LTR's benefit when the number of task increases.
Unlike \LazyPRM, \LTR uses a tree-based method to explore and connects to its persistent graph.
Tree-based methods tend to return a trajectory with a much lower cost due to the locality when growing the trees outwards from their initial and target configurations.
For example, \LazyPRM had spent a substantial amount of time in the PANDA task creating nodes that are not critical to the optimal trajectory (\cref{fig:comp-resulting-graph}). In contrast, tree-based methods can directly bias its tree towards the start or target configuration.
Therefore, \LTR can achieve a much lower cost overall than \LazyPRM (Bottom of~\cref{fig:exp}).

Overall the proposed method \LTR consistently outperforms other baselines on consecutive pick-and-place tasks, where it has the benefit of rapid planning speed of the bidirectional RRT* and the reusability of \LazyPRM.
Interestingly, while there exist certain overheads for \LTR when it checks for connections between the trees $\mathcal{T}_\text{initial}$, $\mathcal{T}_\text{target}$ and the graph $\mathcal{G}$, the quantitative results illustrate that the overall speeds-up overweight the overhead.
The hybrid \LTR grows a tree rapidly to the existing experience graph for a fast solution, and it will keep refining the tree when there is still time budget remained to refine its solution.
In a life-long setting where the robot arm needs to execute an indefinitely long sequence of tasks, the lazy experience graph might become arbitrarily dense.
An overly dense graph helps to ensure completeness and optimality guarantees; however, performing graph search in such a graph might degrade performance.
One possible strategy to mitigate the issue is via witness set~\autocite{li2016asymptotically} which helps to maintain a sparse data structure by pruning nodes from the graph.
This strategy only ensures probabilistic $\delta$-robust completeness; however, it can help control the lazy experience graph's growth over the life-long setting.

\section{Conclusion}
We proposed a hybrid planner as a method for planning specifically in consecutive object manipulation tasks.
The challenge lies in the modified configuration space in-between each task due to the moved objects and the need to rapidly replans while achieving low-cost trajectory.
Our proposed method utilises a lazy experience graph to keep track of past data while performing tree-based planning to rapidly expands the search for low-cost trajectory.
As a result, we are able to retain valuable information in-between each task, while achieving low trajectory cost that is commonly available in a tree-based planner.
Such a hybrid method capitalise the rapid expansion property from tree-based SBPs, and the reusability from the roadmap-based SBPs.
In the future, we can better exploit the change-of-obstacle locality by learning a sampling distribution that concentrates in ``active regions'',
which
can then focus more on regions that require replanning.

\printbibliography

\end{document}